%% file: main_a4.tex
\documentclass[conference,a4paper]{IEEEtran}
\IEEEoverridecommandlockouts
\usepackage{cite}
\usepackage{amsmath,amssymb,amsfonts}
\usepackage{amsthm}
\newtheorem{assumption}{Assumption}
\newtheorem{lemma}{Lemma}
\newtheorem{proposition}{Proposition}
\newtheorem{corollary}{Corollary}
\theoremstyle{remark}
\newtheorem{remark}{Remark}
\usepackage{algorithmic}
\usepackage{graphicx}
\usepackage{float}
\usepackage{hyperref}
\usepackage{subcaption}
\usepackage{tikz}
\usetikzlibrary{positioning,arrows.meta}
\usepackage{textcomp}
\usepackage{circuitikz}
\usetikzlibrary{fit, positioning}
\usepackage[percent]{overpic}  

\usepackage{xcolor}
\newcommand{\blue}[1]{\textcolor{black}{#1}}
\def\BibTeX{{\rm B\kern-.05em{\sc i\kern-.025em b}\kern-.08em
    T\kern-.1667em\lower.7ex\hbox{E}\kern-.125emX}}
\begin{document}

\title{PE-MHL: Physics-Encoded Modular Hybrid Layers for Scalable Learning of 
Complex Systems}

\author{\IEEEauthorblockN{Ismail Hassaballa}
\IEEEauthorblockA{\textit{Eindhoven University of Technology}\\
Eindhoven, Netherlands \\
ismail.hassaballa@student.tue.nl}
\and
\IEEEauthorblockN{Mircea Lazar\textsuperscript{*}}
\IEEEauthorblockA{\textit{Eindhoven University of Technology}\\
Eindhoven, Netherlands \\
m.lazar@tue.nl}
\thanks{\textsuperscript{*}Corresponding author: Mircea Lazar (m.lazar@tue.nl).}
}

\maketitle

\begin{abstract}

Hybrid models that combine physics-based and data‐driven components have shown strong potential for achieving accuracy and interpretability in control applications. While recent methods have made progress in incorporating physical consistency, challenges remain in scalability, robustness to noise, and control of model complexity. This paper proposes a \emph{Physics-Encoded Modular Hybrid Layer} (PE-MHL) framework, in which a baseline physics‐based model is incrementally refined through the addition of new sub-models, where each new component adds complexity while preserving what previous components have already learned. We establish a theoretical guarantee for this construction: with a least-squares initialization of each new sub-model, the training error is monotonically non-increasing in the number of sub-models and provably converges. Empirical evaluations on a nonlinear NARX benchmark and the Quanser Aero 2 platform demonstrate that PE-MHL \blue{matches or outperforms its fully connected monolithic relaxation (over five times more parameters)} in accuracy and generalization, while also providing more stable training dynamics and better preservation of underlying data structures. 

\end{abstract}

\newif\ifl
\lfalse  % set to \longfalse to hide these sections

\section{Introduction}
In modern control applications, there is a growing demand for models with higher accuracy and reliability, as even small prediction errors can result in inadequate or unsafe system behavior. Physics‐based models offer interpretability and adhere to physical laws, but often struggle to capture complex real‐world dynamics. Data‐driven models, on the other hand, can fit complex patterns, but they typically require large datasets and tend to generalize poorly beyond their training data \cite{li2025editorial}.

Hybrid modeling, where a physics‐based component is combined with a data‐driven model, has emerged as a promising approach that leverages the strengths of both. By encapsulating physical knowledge into the model and its learning process, hybrid models can achieve higher predictive accuracy with less data, while maintaining consistency with the underlying physics \cite{karniadakis2021physics, sivaranjani2025control}. Physical knowledge can be imposed either soft-coded, through penalty terms in the training loss as in physics-informed neural networks, or hard-coded, by encoding physical structure directly into the model architecture; the latter has been shown to improve extrapolation and interpretability while avoiding the sensitive hyperparameter tuning of soft-constraint methods \cite{rao2023encoding}. For example, the authors of \cite{bolderman2021pgnn} proposed a Physics-Guided Neural Network (PGNN) architecture in which a physics-inspired model is corrected by a neural network. This approach was later extended in \cite{bolderman2022pgnn} by introducing cost regularization and optimized initialization strategies, resulting in improved performance and interpretability. \cite{jia2021physicsguided} developed a physics‐guided RNN for lake temperature forecasting, achieving over 20\% increase in prediction accuracy with limited training data. More recently, physics-informed neural networks have been applied to the joint state estimation and identification of nonlinear dynamical systems from sparse, noisy measurements \cite{haywoodalexander2025response}. Other applications in health monitoring confirm that neural architectures incorporating physics enhance both robustness and interpretability \cite{wu2024review}.

Despite these successes, several gaps remain. Firstly, scalability to high‐dimensional or multi‐physics systems and the integration of multiple physical domains remain open problems \cite{gupta2024physicsguided}. Secondly, ensuring that each component of a hybrid model remains individually interpretable, rather than allowing a neural network to distort the physics model, has not been fully solved, especially under noisy or uncertain data conditions \cite{wong2022robustness}. Recent work has begun to address this through regularization strategies that constrain the neural network during training to preserve the physical meaning. For example, the method proposed in \cite{bolderman2022pgnn} introduced a cost function that penalizes deviation from known physical parameters, thereby maintaining interpretability in physics-guided neural networks.

To further address these issues, a \emph{progressive hybrid ensemble} framework, termed Physics-Encoded Modular Hybrid Layers (PE-MHL), is developed in this paper. Building on prior hybrid modeling strategies \cite{bolderman2021pgnn, bolderman2022pgnn}, the method incrementally refines the model through the addition of new components, each of which adds complexity while preserving the structure that the previous components have already learned. The sequential refinement of residual errors is a trait shared with boosting and, in particular, with gradient boosting, which fits each new learner to the residuals of the current ensemble \cite{schapire1990strength, friedman2001greedy}. The proposed framework differs by preserving the interpretability of the physics-based core and by anchoring previously trained components through a deviation penalty, rather than freezing or discarding them. This anchoring is conceptually related to continual-learning techniques such as progressive networks \cite{rusu2016progressive} and elastic weight consolidation \cite{kirkpatrick2017overcoming}, which similarly constrain parameters to retain prior knowledge as new capacity is added. In contrast to these somewhat heuristic constructions, we provide theoretical guarantees, i.e., under a mild persistency-of-excitation condition, the least-squares initialization of each new sub-model ensures that adding it can never increase the training error and that the resulting error sequence converges, with geometric decay to zero under an additional condition. The developed framework is evaluated on several benchmarks from the literature, i.e., learning a static nonlinear function from \cite{pillonetto2022regularization} and a  nonlinear auto-regressive exogenous (NARX) system from \cite{billings2013nonlinear}. Furthermore, the approach is validated on a real-world system, the Quanser Aero 2 platform, which is widely used in academic research to validate control algorithms due to its nonlinear dynamics.

%=============================================================
\section{Progressive Physics-Encoded Neural Models}
%=============================================================

The PE-MHL framework builds a hybrid model incrementally by combining a physics-encoded base block with a sequence of lightweight neural sub-models. Each sub-model is trained to capture the residual error left by its predecessors, so that the ensemble grows in complexity only as needed while preserving what earlier components have already learned.

Starting from a physics-encoded base model $\hat{y}_0$ with parameters $\theta_0$, lightweight neural sub-models $\hat{y}_1, \hat{y}_2, \dots, \hat{y}_L$ are appended sequentially. Note that it is also possible to endow the physics-encoded base model with sub-models, which is useful for capturing multi-physics aspects. 

The ensemble prediction at stage $L$ is the sum of all component outputs:
\begin{equation}
    \hat{y}(x) = \sum_{j=0}^{L} \hat{y}_j(x).
\end{equation}
The resulting architecture is illustrated in Fig.~\ref{fig:PEMHL}.

\input{pictures/PEMHL}

Using \(L\) to denote the number of sub-models, let \(\theta_j\) represent the parameters of sub-model \(j\) and \(\hat{y}_j\) its predicted output. The \emph{PE-MHL} is built as follows:\\

\begingroup
\noindent\textbf{Step 0: Initial base model:}\\
\begin{equation}\label{eq:step0}
\theta_0^* \;=\;
\arg\min_{\theta_0}\;
\frac{1}{N}\sum_{i=1}^N \bigl[y_i - \hat{y}_0\bigr]^2
\;+\;\lambda_{\rm ridge}\,\|\theta_0\|^2
\end{equation}

\noindent\textbf{Step 1: Add first neural component:}\\
\begin{equation}\label{eq:step1}
\{\theta_0,\theta_1\}^*
= \arg\min_{\theta_0,\theta_1}\;
\frac{1}{N}\sum_{i=1}^N \bigl[y_i
  - \hat{y}_0
  - \hat{y}_1\bigr]^2
+ \lambda_{\rm dev}\,\|\theta_0-\theta_0^*\|^2
\end{equation}

\noindent\textbf{Step 2: Add second neural component:}\\
\begin{equation}\label{eq:step2}
\begin{split}
\{\theta_j\}_{j=0}^2{}^*
&= \arg\min_{\{\theta_j\}_{j=0}^2}\;
\frac{1}{N}\sum_{i=1}^N \Bigl[y_i
  - \sum_{j=0}^2 \hat{y}_j\Bigr]^2\\&
+ \lambda_{\rm dev}\sum_{j=0}^1\|\theta_j-\theta_j^*\|^2
\end{split}
\end{equation}

\noindent\textbf{Step $L$: General $L$th component:}\\
\begin{equation}\label{eq:stepk}
\begin{split}
\{\theta_j\}_{j=0}^L{}^*
&= \arg\min_{\{\theta_j\}_{j=0}^L}\;
\frac{1}{N}\sum_{i=1}^N \Bigl[y_i
  - \sum_{j=0}^L \hat{y}_j\Bigr]^2\\
&+ \lambda_{\rm dev}\sum_{j=0}^{L-1}\|\theta_j-\theta_j^*\|^2
\end{split}
\end{equation}
\endgroup

The loss function for training each new component is:
\begin{equation}
  \textit{Loss}
  = \frac{1}{N}\sum_{i=1}^N\Bigl[y_i - \sum_{j=0}^L \hat{y}_j\Bigr]^2
  + \lambda_{\mathrm{dev}}\sum_{j=0}^{L-1}\|\theta_j - \theta_j^*\|_2^2,
\end{equation}
where \(\hat{y}_j\) is the output of the \(j\)-th model with parameters \(\theta_j\), \(L\) is the total number of models at the current step (including the new one), \(\theta_j^*\) are the frozen parameters of model \(j\) from the previous stage, and \(\lambda_{\text{dev}}\) controls the permitted deviation of prior models from their optimal solutions.

The first term minimizes the combined prediction error, while the second term penalizes large changes to previously trained models, ensuring each new component focuses on the residual structure not yet captured by the ensemble.

In the two-component setting (physics base model + one neural network), training with MSE alone causes the two branches to act as inverses of each other without learning individually meaningful structures. This is particularly problematic for physics-guided models, where each component should retain physical interpretability. To enforce this, a penalty term constrains the physics-based component to remain close to its initial Ridge-EVE estimate $\theta_{\mathrm{ridge}}$:
\begin{equation}
    L_{\mathrm{penalty}} = \lambda \| \theta_{\mathrm{poly}} - \theta_{\mathrm{ridge}} \|^2,
\end{equation}
yielding the total loss:
\begin{equation}
    L_{\mathrm{total}} = \mathrm{MSE} + L_{\mathrm{penalty}},
\end{equation}
\blue{where $\mathrm{MSE}=\frac{1}{N}\sum_{i=1}^{N}\big(y_i-\hat{y}(x_i)\big)^2$ is the mean-squared prediction error over the $N$ training points.}
This penalty is a special case of the deviation regularization in the general PE-MHL loss, applied at the first stage to anchor the physics component.

%=============================================================
\section{Monotonicity and Convergence of the PE-MHL Residual}
\label{sec:convergence}
%=============================================================

We now show that the progressive training~\eqref{eq:stepk}, when each new
sub-model is initialized through its output layer, yields a training error
that is non-increasing in the number of sub-models and that converges. The
result extends the single-step initialization guarantee
of~\cite{bolderman2022pgnn} from one physics--neural comparison to an
arbitrarily long modular sequence.

For the stage-$L$ ensemble with frozen predecessor optima
$\{\theta_j^*\}_{j=0}^{L-1}$, define the residual and the training error
\begin{equation}
\label{eq:residual}
\begin{split}
  r_i^{(L)} &:= y_i - \sum_{j=0}^{L}\hat{y}_j(x_i;\theta_j^*),\\
  \mathcal{E}_L &:= \frac{1}{N}\sum_{i=1}^N\!\big(r_i^{(L)}\big)^2
  = \frac{1}{N}\big\|r^{(L)}\big\|_2^2,
\end{split}
\end{equation}
where $r^{(L)}=[r_1^{(L)},\dots,r_N^{(L)}]^\top$ and $\mathcal{E}_0$ is the
mean--squared error of the physics-encoded base model~\eqref{eq:step0}.

Following~\cite{bolderman2022pgnn}, write the new sub-model with an affine
output layer
\begin{equation}\label{eq:affine_layer}
\begin{split}
  \hat{y}_L(x;\theta_L)&=\theta_{{\rm OL},L}^\top\,\phi_{{\rm OL},L}(x),\\
  \phi_{{\rm OL},L}(x)&=\big[\,f_{{\rm HL},L}(\theta_{{\rm HL},L},x)^\top,\;1\,\big]^\top,
  \end{split}
\end{equation}
with $\theta_L=\{\theta_{{\rm HL},L},\theta_{{\rm OL},L}\}$, where
$f_{{\rm HL},L}:\mathbb{R}^{n_x}\!\to\mathbb{R}^{n_L}$ is the last hidden layer
($n_L$ neurons) and $\theta_{{\rm OL},L}\in\mathbb{R}^{n_L+1}$ collects the
output weights and bias. Let $\Phi_L\in\mathbb{R}^{N\times(n_L+1)}$ have rows
$\phi_{{\rm OL},L}(x_i)^\top$, let $P_{\Phi_L}$ denote the orthogonal
projector onto $\operatorname{col}(\Phi_L)$, and set
$M_L:=\tfrac{1}{N}\Phi_L^\top\Phi_L$.

\begin{assumption}\label{asn:pe}
For a fixed random draw of $\theta_{{\rm HL},L}$, the data are persistently
exciting for the output layer of sub-model $L$, i.e., $M_L$ is non-singular.
\end{assumption}

Then we can state the following result.

\begin{lemma}\label{lem:decrease}
Let Assumption~\ref{asn:pe} hold. Freeze the predecessors at
$\{\theta_j^*\}_{j=0}^{L-1}$, fix $\theta_{{\rm HL},L}$ at its random
initialization, and initialize the output layer by least squares,
\begin{equation}\label{eq:ls_init}
  \theta_{{\rm OL},L}^{(0)}
  = M_L^{-1}\!\left(\frac{1}{N}\sum_{i=1}^N
      \phi_{{\rm OL},L}(x_i)\,r_i^{(L-1)}\right).
\end{equation}
Then $\hat{y}_L=P_{\Phi_L}r^{(L-1)}$ and
\begin{equation}\label{eq:pythagoras}
  \mathcal{E}_L
  = \mathcal{E}_{L-1}-\frac{1}{N}\big\|P_{\Phi_L}r^{(L-1)}\big\|_2^2
  \;\le\;\mathcal{E}_{L-1},
\end{equation}
with strict inequality if and only if
$\tfrac{1}{N}\sum_{i}\phi_{{\rm OL},L}(x_i)\,r_i^{(L-1)}\neq 0$.
\end{lemma}
\begin{proof}
The choice $\theta_{{\rm OL},L}=0$ gives $\hat{y}_L\equiv0$; with the
predecessors at $\theta_j^*$ the deviation penalty in~\eqref{eq:stepk}
vanishes and the cost equals $\mathcal{E}_{L-1}$, so this point is feasible.
With all other parameters frozen, the cost~\eqref{eq:stepk} is quadratic in
$\theta_{{\rm OL},L}$ and, by Assumption~\ref{asn:pe}, has the unique global
minimizer~\eqref{eq:ls_init}, whose fitted output is the orthogonal
projection $\hat{y}_L=P_{\Phi_L}r^{(L-1)}$. Hence, the updated residual is the
complementary projection $r^{(L)}=(I-P_{\Phi_L})r^{(L-1)}$. Since $P_{\Phi_L}$
is symmetric and idempotent, these two components are orthogonal, so expanding
$\|r^{(L-1)}\|^2$ gives the Pythagorean identity
$\|r^{(L-1)}\|^2=\|P_{\Phi_L}r^{(L-1)}\|^2+\|(I-P_{\Phi_L})r^{(L-1)}\|^2$,
which yields~\eqref{eq:pythagoras}. The decrease is strict iff
$P_{\Phi_L}r^{(L-1)}\neq0$, i.e., iff the residual is correlated with at least
one new feature.
\end{proof}

\begin{remark}\label{rem:free}
Since the predecessors $\{\theta_j\}_{j=0}^{L-1}$ are not frozen but only
penalized in~\eqref{eq:stepk}, the bound~\eqref{eq:pythagoras} transfers to
the free problem by feasibility: the point obtained by holding the
predecessors at $\{\theta_j^*\}_{j=0}^{L-1}$ and setting
$\theta_{{\rm OL},L}=\theta_{{\rm OL},L}^{(0)}$ attains cost
$\mathcal{E}_{L-1}-\tfrac1N\|P_{\Phi_L}r^{(L-1)}\|_2^2$ with zero deviation
penalty, so any best-retained iterate of~\eqref{eq:stepk} attains a no larger cost. As the penalty is non-negative, the resulting $\mathcal{E}_L$
in~\eqref{eq:residual} -- now evaluated at the \emph{updated} predecessor
parameters -- still satisfies~\eqref{eq:pythagoras} with non-strict
inequality, and with strict inequality whenever the best iterate improves on
this feasible point. 
\end{remark}

Next, we use the above result to analyze the convergence of the PE-MHL modular architecture.

\begin{proposition}\label{prop:converge}
Under Assumption~\ref{asn:pe} at every stage, the training errors
$\{\mathcal{E}_L\}_{L\ge0}$ produced by~\eqref{eq:stepk} with the
initialization~\eqref{eq:ls_init} are non-increasing and bounded below by zero,
so they converge to a positive limit, i.e.,
\begin{equation}\label{eq:limit}
  \mathcal{E}_L\;\searrow\;\mathcal{E}_\infty\ge0 .
\end{equation}
Moreover, the error reduction contributed by each new sub-model,
$\Delta_L:=\tfrac{1}{N}\|P_{\Phi_L}r^{(L-1)}\|_2^2$, must shrink to zero, i.e.,
\begin{equation}\label{eq:summable}
  \sum_{L\ge1}\Delta_L \;=\;\mathcal{E}_0-\mathcal{E}_\infty\;<\;\infty,
  \qquad\text{hence}\qquad \Delta_L\to0 .
\end{equation}
\end{proposition}
\begin{proof}
Each stage can only lower the error (Lemma~\ref{lem:decrease} and
Remark~\ref{rem:free}), bounded below by zero; a sequence
that decreases but cannot go below a floor must converge, which
gives~\eqref{eq:limit}. Summing the per-stage reductions~\eqref{eq:pythagoras}
yields the total drop $\mathcal{E}_0-\mathcal{E}_\infty$, which is
finite. Thus, a sum of non-negative terms that is finite forces those terms, $\Delta_L$, to vanish.
\end{proof}

A stronger convergence result can be obtained under a stronger progress condition as follows.
\begin{corollary}
\label{cor:rate}
Suppose each sub-model removes at least a fixed fraction $\gamma\in(0,1]$ of
the error remaining before it was added,
\begin{equation}\label{eq:weakgreedy}
  \Delta_L \;\ge\; \gamma\,\mathcal{E}_{L-1},
  \qquad\forall\,L\ge1 .
\end{equation}
Then $\mathcal{E}_L\le(1-\gamma)^L\,\mathcal{E}_0\to0$, i.e. the error decays
geometrically to zero.
\end{corollary}

\begin{proof}
Condition~\eqref{eq:weakgreedy} in~\eqref{eq:pythagoras} gives
$\mathcal{E}_L\le(1-\gamma)\mathcal{E}_{L-1}$; applying this from
$\mathcal{E}_0$ onward yields the bound.
\end{proof}
\ifl
\begin{remark}\label{rem:interpret}
The two above results provide complementary analyses. Proposition~\ref{prop:converge}
guarantees, with no extra assumptions, that stacking sub-models always helps
and the process settles down: each addition reduces the error and the gains
fade out, which is the theoretical counterpart of the diminishing returns
observed in Fig.~\ref{fig:mae_model} in the Results section, and motivates a finite stopping rule.
What it does \emph{not} promise is that the error reaches zero. For that,
Corollary~\ref{cor:rate} requires every sub-model to keep capturing a
non-vanishing share of the leftover error, which in turn needs the sub-models
to be rich enough to represent the part of $y$ still unexplained. Where the
data hold structure cannot be expressed by no other small sub-model, the fraction
$\gamma$ collapses and the error plateaus above zero. 
\end{remark}
\fi
\begin{remark}
\label{rem:interpret}
Every $\mathcal{E}_L$ above is measured on the training data. Driving
it to zero would mean fitting the training points exactly, including their
measurement noise, which worsens accuracy on
unseen data. The quantity that matters in practice is therefore the test
error, and the deviation penalty $\lambda_{\rm dev}$ (which keeps each module
from fitting noise) together with stopping at a finite number of
sub-models are what keep the model general, rather than minimizing training
error alone.
\end{remark}

%=============================================================
\section{Implementation in Benchmark Problems}
%=============================================================

The PE-MHL framework is evaluated on three benchmarks of increasing complexity. The general architecture and training procedure described in Section~II apply throughout; this section details the specific modeling choices for each case.

\subsection{Static Function}

A nonlinear static function from \cite{pillonetto2022regularization} is used to generate data, defined as:
\begin{equation}
    g(x) = \sin^2(x)(1 - x^2), \quad x \in [0, 1],
\end{equation}
with additive Gaussian noise:
\begin{equation}
    y = g(x) + e,
\end{equation}
where \blue{\( e \) is zero-mean noise, uniformly distributed on $[-0.04,0.04]$ (theoretical variance $5.33\times10^{-4}$); the empirical variance of the realized noise sequence is $4.68\times10^{-4}$.} 

\subsubsection{Base Model}

A third-degree polynomial is used as the physics-inspired base model, expressed as:
\begin{equation}
    y_i = \theta_1 + \sum_{k=2}^{4} \theta_k x_i^{k-1} = \theta_1 + \theta_2 x_i + \theta_3 x_i^2 + \theta_4 x_i^3, \,\, i = 1, \dots, N,
\end{equation}
where \( x_i \) are input features, \( y_i \) noisy observations, \( \theta=[\theta_1, \theta_2, \theta_3, \theta_4]^T \) model parameters, and \( N \) the number of training points.

In matrix form \( Y = X\theta \), the least squares solution is:
\begin{equation}
\hat{\theta} = \arg\min_{\theta} \| Y - X \theta \|_2^2.
\end{equation}

\subsubsection{Base Model Debiasing}

To improve the robustness of the base polynomial fit, several debiasing techniques were considered. Ridge regression with $\ell_2$ regularization, tuned using the Expected Validation Error (EVE) criterion \cite{pillonetto2022regularization}, provided the most reliable results.

\subsubsection{Hybrid Architecture: Polynomial + Neural Network}
Following debiasing, a two-branch hybrid is constructed:
\begin{itemize}
    \item \textbf{Polynomial Branch:} A 3rd-degree polynomial capturing low-complexity patterns, playing the role of the physics-based model in PGNN frameworks.
    \item \textbf{Neural Network Branch:} A feedforward network with three hidden layers of 50 ReLU neurons each, designed to learn the residual complexity.
\end{itemize}
The combined output is (see Fig.~\ref{fig:architecture}):
\begin{equation}
    \hat{y}(x) = y_{poly}(x) + y_{nn}(x).
\end{equation}

\begin{figure}[htbp]
  \centering
  \resizebox{0.85\columnwidth}{!}{%
    \begin{tikzpicture}[
        >=Latex,
        font=\normalsize,
        node distance=6mm and 12mm
      ]
      \coordinate (input) at (0,0);
      \node (poly) [rectangle, draw, fill=green!20,
                    right=2cm of input, yshift=12mm,
                    minimum width=5.5cm, minimum height=1cm,
                    align=center] {
        \textbf{Polynomial 3\textsuperscript{rd} degree}\\
        $y_{\text{poly}} = \theta_{1} + \theta_{2}x + \theta_{3}x^{2} + \theta_{4}x^{3}$
      };
      \node (nn) [rectangle, draw, fill=blue!10,
                  below=of poly,
                  minimum width=5.5cm, minimum height=1.5cm,
                  align=center] {
        \textbf{Neural Network}\\
        3 layers \texttimes\ 50 ReLU neurons
      };
      \node (sum) [circle, draw, right=12mm of nn.north] at ($(poly.east)!0.5!(nn.east)$) {$+$};
      \draw[->] (sum.east) -- ++(8mm,0) node[right] {$y$};
      \draw[->] (input) -- ++(8mm,0) |- (poly.west);
      \draw[->] (input) -- ++(8mm,0) |- (nn.west);
      \node[anchor=east] at (input) {\(x\)};
      \draw[->] (poly.east) -- (sum.west);
      \draw[->] (nn.east) -- (sum.west);
      \node[draw=black, dashed, fit=(poly)(nn), inner sep=8pt, label={[yshift=0.4cm]above:\textbf{Hybrid Model}}] (hybrid) {};
    \end{tikzpicture}
  }
  \caption{Architecture of the hybrid system with a 3\textsuperscript{rd}-degree polynomial and a neural-network model.}
  \label{fig:architecture}
\end{figure}

\begin{figure}[htbp]
  \centering
  \begin{overpic}[
      width=0.78\columnwidth,
      trim=0 0 0 1.4cm,
      clip,
      percent
    ]{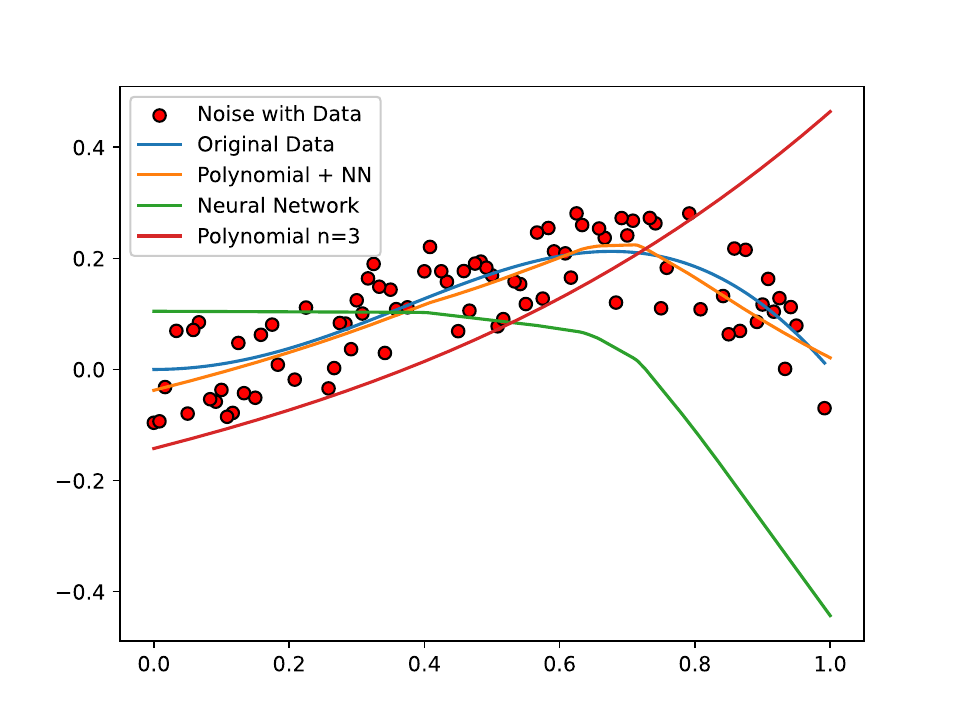}
    \put(44,1){\small $x$ value}
    \put(3,30){\rotatebox{90}{\small y value}}
  \end{overpic}
  \caption{Training without penalty: the total output fits the data, but individual model components behave as inverses and lack meaningful structure.}
  \label{fig:bad_decomposition}
\end{figure}

\subsection{NARX Dynamical System}

The PE-MHL architecture is evaluated on a nonlinear autoregressive model with exogenous input (NARX), originally introduced by Billings \cite{billings2013nonlinear}:
\begin{equation}
\begin{split}
y(k) &= 0.5\,y(k-1) + 0.3\,u(k-1) \\
     &\quad +\,0.3\,u(k-1)\,y(k-1) + 0.5\,u^2(k-1)\,,
\end{split}
\label{eq:NARX}
\end{equation}
where \(k\in\{1,\dots,N\}\) is the discrete time index, \(y(k)\in\mathbb{R}\) the system output, \(u(k)\in\mathbb{R}\) the input, and \(y(k-1)\), \(u(k-1)\) the one-step lagged output and input respectively.

\subsubsection{Monolithic Baseline}

\blue{To isolate the effect of progressive stacking, PE-MHL is compared against its fully connected relaxation: viewing the $L=6$ width-$50$ sub-models as a block-diagonal network of total width $300$, the monolithic baseline is obtained by additionally allowing all inter-module connections. The resulting parameter count is $272{,}101$ for the monolithic network versus $47{,}106$ for PE-MHL, i.e., the surplus of roughly $225$k parameters corresponds precisely to the inter-module weights.}

The monolithic network  takes as input a bias term, \(y(k-1)\), and \(u(k-1)\), passes through three hidden layers of 300 ReLU neurons each,
\begin{equation}
    \mathrm{ReLU}(z)=\max(0,z),
\end{equation}
and produces \(\hat{y}(k)\) via a single linear output neuron. Training minimizes MSE using the Adam optimizer.
\input{pictures/arch}

\subsubsection{System Excitation Data}
\blue{The identification data set consists of $3998$ input--output samples generated by a multisine excitation $u(k) = \sum_{i=1}^{4} \sin(2\pi f_i k + \phi_i)$, with normalized frequencies \(f_i \in \{0.02, 0.05, 0.10, 0.20\}\) and random phases, scaled to \([-1,1]\). The regressor is restricted to the causal lags $[u(k-1),\,y(k-1)]^\top$, consistent with \eqref{eq:NARX}; all runs are seeded and reproducible. Trained models are evaluated on four unseen test inputs: a unit step (an extrapolation test, as it drives the output to $y_{ss}=4$ while training outputs do not exceed $3.37$), a uniform random input $u(k)\sim\mathcal{U}[-1,1]$, a multisine with unseen phases (also the validation signal during training), and a staircase.}

\subsection{Experimental Setup: Quanser Aero 2}
%============================================================%

The proposed identification approach is evaluated on the Quanser Aero 2 (Fig.~\ref{fig:quanser}), a lab-scale, two-degree-of-freedom (DOF) helicopter model used in control research.  In this setup, the \emph{yaw} dynamics are isolated by locking the pitch and allowing rotation only about the vertical axis.

\begin{figure}
    \centering
    \includegraphics[width=0.42\linewidth]{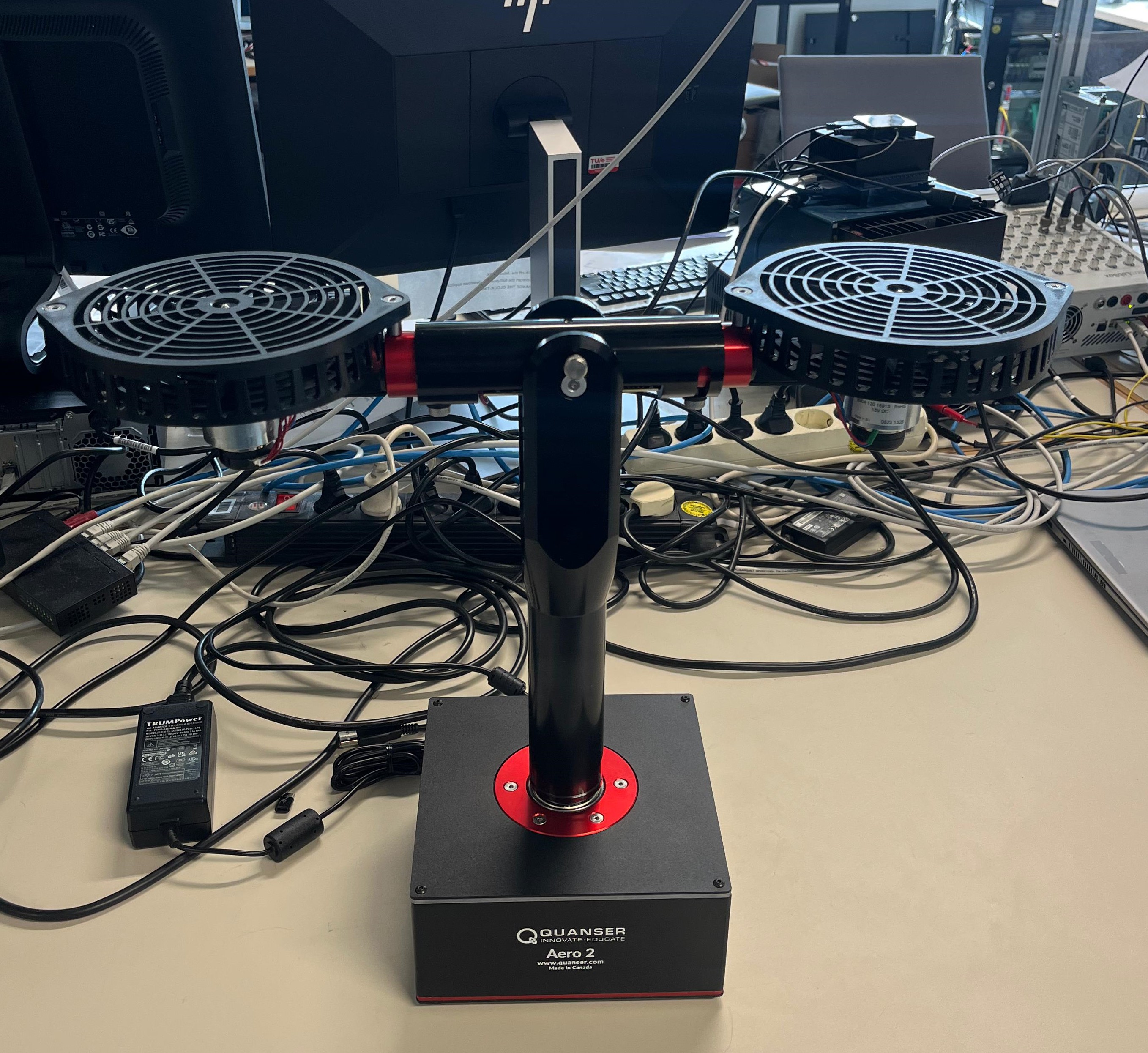}
    \caption{The Quanser Aero 2 system}
    \label{fig:quanser}
\end{figure}

The Aero 2 uses two direct-drive, brushed DC motors (±18 V rating).  A differential drive configuration is implemented for yaw control, in which the input voltage \(u(t)\) is split into two branches:
\begin{enumerate}

    \item  The first branch inverts the voltage (multiplies by -1) and applies it to Motor 1.
    \item  The second branch applies the voltage directly to Motor~2.

\end{enumerate}

\input{pictures/aero}

This inversion causes the two motors to rotate in opposite directions, producing a yaw motion around the vertical axis.

To excite the yaw dynamics across a range of frequencies, a three-tone multisine input signal is applied:

\begin{equation}
\begin{split}
u(t) &= 5\,\sin\bigl(2\pi \times 0.01\,t\bigr)
      + 3\,\sin\Bigl(2\pi \times 0.05\,t + \tfrac{\pi}{3}\Bigr)\\
     &\quad\;+\;10\,\sin\Bigl(2\pi \times 0.1\,t + \tfrac{\pi}{2}\Bigr)\,,
\end{split}
\end{equation}
where \(t\) denotes the time in seconds. The signal is sampled at a fixed sampling time of \(\,T_s=0.002\)s. The resulting yaw angle \(y(t)\) is recorded for system identification.

The resulting data is split into training set and test subset to allow for validation on unseen data. For ARX feature construction, \(n_u=23\) past inputs and \(n_y=15\) past outputs are used, yielding an input vector of
\begin{equation*}
\begin{split}
\mathbf{x}(k) = \bigl[u(k),\dots,u(k-23),\,y(k-1),\dots,y(k-15)\bigr]^T.
    \end{split}
\end{equation*}

As a baseline for comparison, a single monolithic neural network is trained under the same conditions as the PE-MHL model. The total number of epochs is identical, and the hidden layer width of the baseline network matches the total number of hidden neurons across all sub-models in the PE-MHL model.

\section{Results}

\subsection{Static Function}

\subsubsection{Two-Model Hybrid Results}

Figure~\ref{fig:good_decomposition} shows the final output of the two-model hybrid system after training with penalties. The polynomial branch preserves its overall shape and remains close to the original de-biased polynomial fit. The neural network branch primarily captures the residual errors not modeled by the polynomial component.

In contrast, when trained without penalties (Fig.~\ref{fig:bad_decomposition}), the two branches interfere with one another, resulting in less interpretable individual contributions. In the penalized case, both components contribute distinct and meaningful structures to the final prediction.

\begin{figure}[htbp]
  \centering
  \begin{overpic}[
      width=0.75\columnwidth,
      trim=0 0 0 1.4cm,
      clip,
      percent
    ]{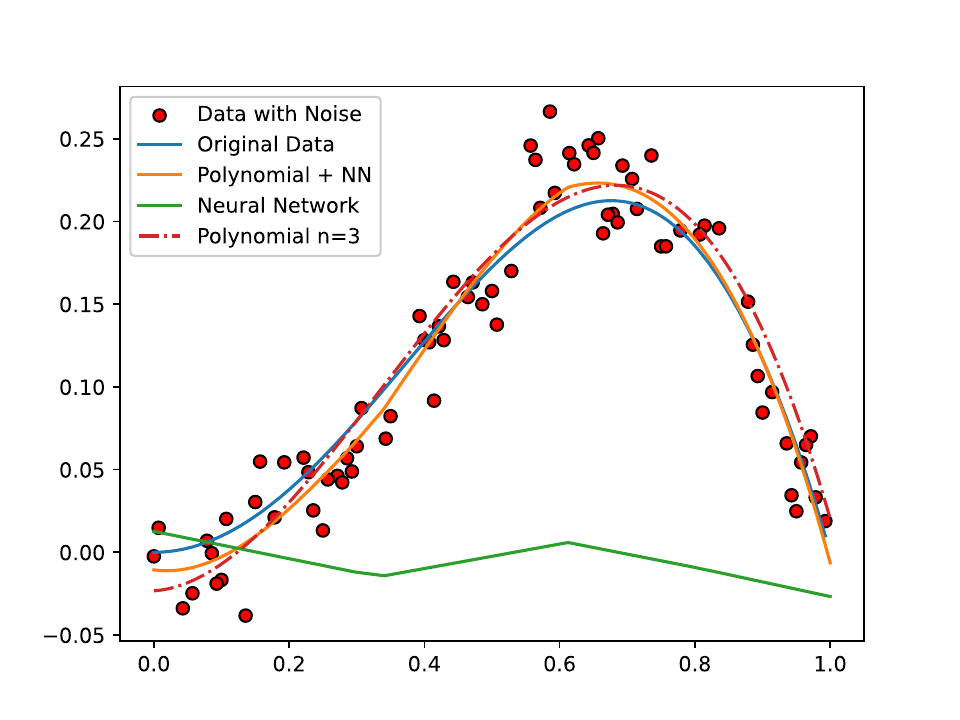}
        \put(44,1){\small $x$ value}
        \put(2,31){\rotatebox{90}{\small y value}}
  \end{overpic}
  \caption{Training with penalties: both branches learn distinct, meaningful structures.}
  \label{fig:good_decomposition}
\end{figure}

\subsection{PE-MHL Results on the NARX example}
\blue{Figure~\ref{fig:narx_convergence} shows the training and validation MSE versus epoch at identical gradient-step counts ($15{,}000$ full-batch Adam iterations, equal learning rates). Each PE-MHL stage addition (vertical lines) lowers the ensemble error below the previous plateau, in accordance with Proposition~\ref{prop:converge}; the drift of previously trained sub-models stays below $6\times10^{-6}$. The monolithic network attains a slightly lower final training MSE ($1.97\times10^{-4}$ versus $2.47\times10^{-4}$), but PE-MHL achieves a $6\times$ lower validation error ($1.83\times10^{-5}$ versus $1.06\times10^{-4}$), indicating that the modular structure acts as a regularizer, and trains $3.6\times$ faster ($342$\,s versus $1218$\,s, single CPU).}

\begin{figure}[!ht]
    \centering
    \includegraphics[width=0.72\linewidth]{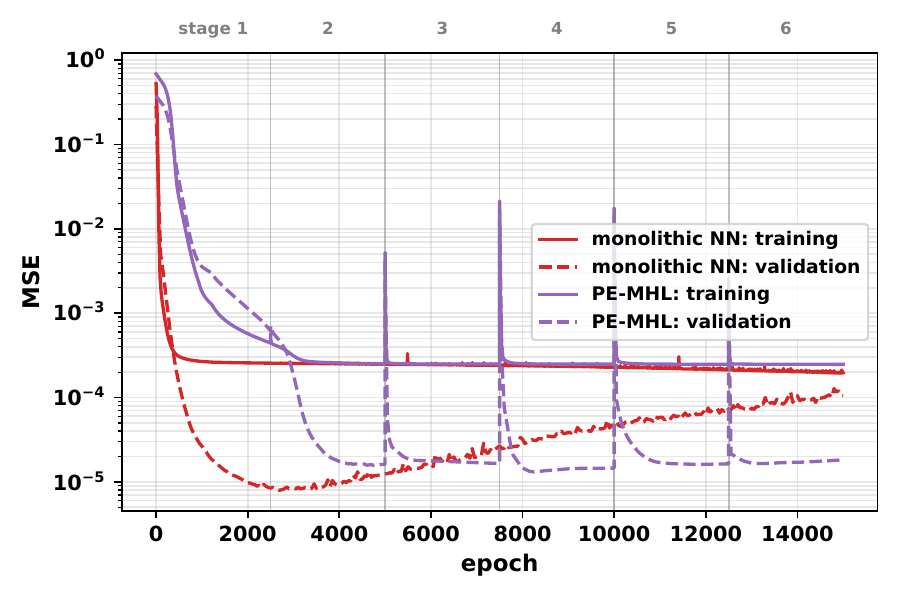}
    \caption{\blue{Training (solid) and validation (dashed) MSE versus epoch for the monolithic network and PE-MHL on the NARX example (validation: one-step prediction on an unseen multisine realization). Vertical lines mark the addition of PE-MHL sub-models.}}
    \label{fig:narx_convergence}
\end{figure}

\blue{Figure~\ref{fig:narx_eval} presents free-run (simulated) responses on the four test inputs, where both models are simulated recursively from their own predictions. On the unit step, the true response settles at $4.0$, outside the range of the training outputs; the monolithic baseline settles at $3.85$ and PE-MHL at $3.79$, i.e., both extrapolate to within approximately $5\%$ of the true steady state. On the remaining tests (MSEs in the caption of Fig.~\ref{fig:narx_eval}), PE-MHL matches or outperforms its fully connected relaxation on the in-distribution signals, and both models remain accurate on the uniform input.}

\begin{figure}[!ht]
    \centering
    \includegraphics[width=0.88\linewidth]{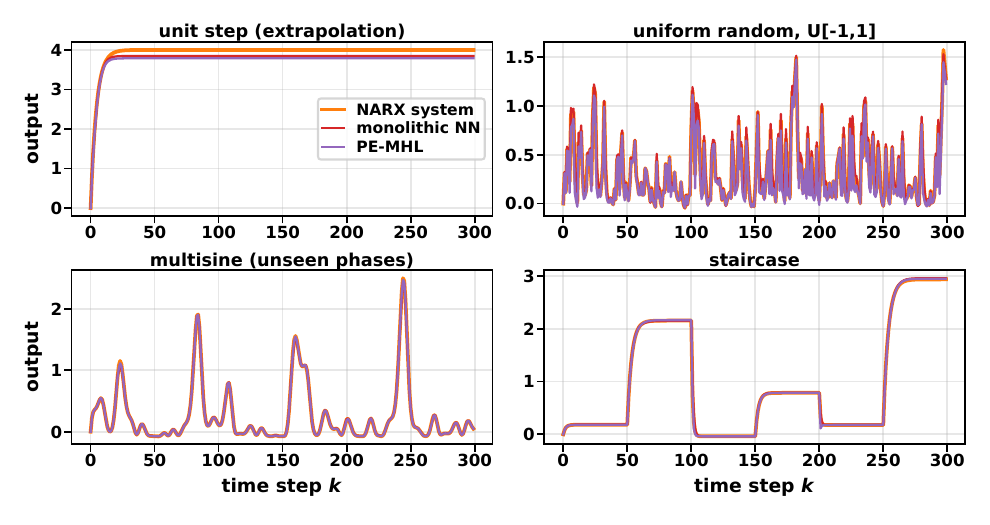}
    \caption{\blue{Free-run responses of the trained models on four test inputs not seen during training. MSE (monolithic vs.\ PE-MHL): uniform $1.38\times10^{-3}$ vs.\ $3.34\times10^{-3}$; multisine $1.68\times10^{-4}$ vs.\ $4.06\times10^{-5}$; staircase $3.82\times10^{-4}$ vs.\ $2.70\times10^{-4}$.}}
    \label{fig:narx_eval}
\end{figure}

\subsection{Experimental validation on the Aero 2 system}

\begin{figure}[!ht]
    \centering
    \includegraphics[width=0.85\linewidth]{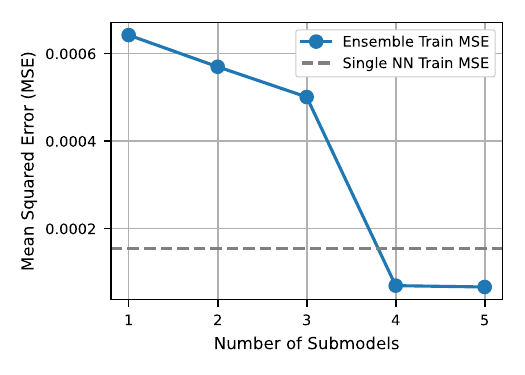}
    \caption{Mean Squared Error after adding each sub-model on the training data}
    \label{fig:mae_model}
\end{figure}

Figure~\ref{fig:mae_model} shows the Mean Squared Error on the \emph{training} data decreasing progressively with the addition of each sub-model: the largest gains occur in the first three submodels (MSE falls from \(6.64\times10^{-5}\) to \(4.98\times10^{-5}\)), while beyond \(L=4\) the curve flattens, indicating diminishing returns; thus, \(L=4\) is a practical choice for the number of sub-models.

This behavior matches the analysis of Section~\ref{sec:convergence}. The monotonic decrease and the necessarily shrinking per-stage gains are guaranteed by Proposition~\ref{prop:converge}, while the flattening beyond \(L=4\) corresponds to the regime where a fresh sub-model can no longer capture a meaningful fraction of the remaining residual -- that is, where the geometric-decrease condition of Corollary~\ref{cor:rate} ceases to hold -- providing a principled justification for stopping at a finite number of sub-models.

\begin{figure}[!ht]
    \centering
        \includegraphics[width=0.92\linewidth,trim=0 0 0 0.15cm,clip]{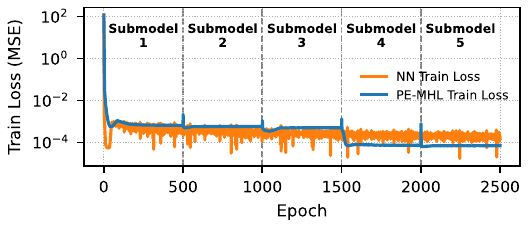}
 
    \caption{Training loss trajectories for the PE‐MHL ensemble and a monolithic NN baseline.  Each vertical dashed line marks the beginning of a new PE‐MHL submodel (every 500 epochs).}
    \label{fig:loss_comparison}
\end{figure}

Figure~\ref{fig:loss_comparison} shows the difference in optimization dynamics between PE-MHL and a single large network.  The PE-MHL’s loss (blue curve) declines in a smooth fashion across each 500-epoch submodel block, whereas the monolithic neural network shows larger oscillations and slower convergence.  This stability, despite both having identical learning rates, indicates better conditioned residual learning in PE-MHL.

\begin{table}[t]
  \centering
  \caption{Final error comparison between the monolithic neural network and the PE-MHL architecture}
  \label{tab:final_errors}
  \begin{tabular}{lcc}
    \hline
    Metric       & Monolithic NN  & PE-MHL \\ \hline
    Train MAE    & $1.12\times10^{-2}$    & $3.37\times10^{-3}$       \\
    Train MSE    & $1.55\times10^{-4}$   & $4.98\times10^{-5}$       \\
    Test MAE     & $1.75\times10^{-2}$   & $3.17\times10^{-3}$       \\
    Test MSE     & $3.35\times10^{-4}$   & $1.54\times10^{-5}$       \\ \hline
  \end{tabular}
\end{table}

Table~\ref{tab:final_errors} summarizes the final performance. Notably, on the unseen test set, the PE-MHL ensemble achieves a test MAE of \(3.17 \times 10^{-3}\) and test MSE of \(1.54 \times 10^{-5}\), substantially lower than the monolithic network’s test MAE of \(1.75 \times 10^{-2}\) and test MSE of \(3.35 \times 10^{-4}\). For PE-MHL the test errors are marginally lower than the training errors; this is consistent with the test signal exciting a less demanding portion of the dynamics than the training signal, together with the deviation penalty $\lambda_{\rm dev}$ acting only on the training objective and thereby slightly raising the reported training fit relative to the unconstrained test evaluation.

\section{Discussion of the results}

\subsection{Static nonlinear function}

The two-model hybrid architecture, combining a debiased polynomial model with a neural network, proved effective in capturing both simple and complex patterns in the data. Crucially, the inclusion of a penalty term in the loss function was essential for preserving the interpretability of the physics inspired polynomial branch. This penalty constrained the polynomial model to remain close to its initial Ridge-EVE estimate, ensuring it retained its physically meaningful structure while allowing the neural network to learn the structures of the data that the polynomial model could not capture. Thus, the penalty based regularization is deemed essential for achieving both accuracy and interpretability in the hybrid system.

\subsection{NARX benchmark and Aero 2 Experimental Setup}

Empirical results from the NARX benchmark and Quanser Aero 2 experiments confirm the effectiveness of the progressive PE-MHL approach. Most performance gains were achieved within the first few sub-models, with diminishing returns beyond that, indicating an effective natural stopping criterion. This empirical pattern is precisely what the convergence analysis of Section~\ref{sec:convergence} predicts: by Proposition~\ref{prop:converge} the training error decreases monotonically and the marginal benefit of each added sub-model necessarily diminishes, so the observed plateau is a principled rather than incidental stopping point. The sustained improvement over the first few sub-models corresponds to the geometric-decay regime of Corollary~\ref{cor:rate}, which holds only while each sub-model still captures a non-negligible fraction of the remaining residual. The better performance of PE-MHL may be attributed to its smoother training dynamics and improved conditioning during optimization, compared to the monolithic neural network baseline as seen in Fig.~\ref{fig:loss_comparison}.

The architecture also introduces a form of self-correction where minor errors introduced by early components can be refined by later ones, enhancing the robustness of the system.

\blue{Regarding subsequent control design, PE-MHL models are directly applicable as prediction models in nonlinear model predictive control: successive linearization can be performed module-wise, sub-models can be selectively refined as new data arrives without retraining the complete model, and the monotone convergence established in Proposition~\ref{prop:converge} provides a constructive handle on the prediction-error bounds required by robust predictive control formulations; a detailed treatment is part of future work.}

\section{Conclusions}

This paper introduced a progressive hybrid modeling framework that combines physics-encoded models with data-driven learning to improve accuracy and interpretability in system identification, control, and other data-driven modeling tasks. The proposed PE-MHL (Physics-Encoded Modular Hybrid Layers) architecture incrementally adds lightweight sub-models, each trained to model residual errors. This structure adapts model complexity to the task, supports natural stopping criteria, and leads to smoother convergence and improved robustness during training. Beyond these empirical benefits, we established a theoretical guarantee for the framework: with a least-squares initialization of each new sub-model, the training error is monotonically non-increasing in the number of sub-models and provably converges, with geometric decay to zero under a weak-greedy condition. This result generalizes the single-step initialization guarantee of \cite{bolderman2022pgnn} to an arbitrarily long modular sequence and formally explains the diminishing-returns behavior that motivates a finite stopping rule. \blue{Compared to its fully connected relaxation, a monolithic neural network with $5.8\times$ more parameters, PE-MHL demonstrated comparable or higher accuracy, better generalization, and faster, more stable training.} Its modular nature ensures that each sub-model refines the residual structure without disrupting prior learning, making the method well suited for low-data or noisy environments. \blue{Extending PE-MHL with inter-module connectivity and adding Gaussian-process and first-principles baselines is part of ongoing work towards a journal version.}

\def\IEEEbibitemsep{0pt plus .5pt}
\bibliographystyle{IEEEtran}

\bibliography{references}

\end{document}

%% file: pictures/PEMHL.tex
\begin{figure}[!ht]
\centering
\resizebox{0.42\textwidth}{!}{%
\begin{circuitikz}
\tikzstyle{every node}=[font=\LARGE]
\draw  (1.75,12.75) rectangle (3.75,7);
\draw [ fill={rgb,255:red,206; green,231; blue,255} ] (4.5,12.75) rectangle  node {\LARGE \shortstack{Physics Encoded\\ \normalsize (base model $f_0$, fixed structure)}} (9.5,11.5);
\draw [ fill={rgb,255:red,165; green,255; blue,138} ] (4.5,10.75) rectangle  node {\LARGE \shortstack{$NN_1$\\ \normalsize (3 layers $\times$ 50 ReLU)}} (9.5,9.75);
\draw [ fill={rgb,255:red,165; green,255; blue,138} ] (4.5,8) rectangle  node {\LARGE \shortstack{$NN_L$\\ \normalsize (3 layers $\times$ 50 ReLU)}} (9.5,7);
\draw [short] (4.5,12.25) -- (3.75,12.25);
\draw [short] (4.5,7.5) -- (3.75,7.5);
\draw [short] (4.5,10.25) -- (3.75,10.25);
\draw  (10.25,12.75) rectangle (12.25,7);
\draw [short] (9.5,10.25) -- (10.25,10.25);
\draw [short] (9.5,7.5) -- (10.25,7.5);
\draw [short] (9.5,12.25) -- (10.25,12.25);
\node [font=\Large, rotate around={90:(0,0)}] at (6.75,9) {. . . . . };
\draw [dashed] (4,13) -- (4,6);
\draw [dashed] (4,6) -- (10,6);
\draw [dashed] (10,6) -- (10,13.25);
\draw [dashed] (9.75,13.25) -- (4,13.25);
\node [font=\Large] at (7,6.5) {Hidden Layers};
\node [font=\LARGE] at (11.25,10) {Output
};
\node [font=\LARGE] at (2.75,10) {Input};
\node [font=\LARGE] at (2.75,9.25) {layer};
\node [font=\LARGE] at (11.25,9.25) {layer};
\draw [->, >=Stealth] (0.25,10) -- (1.75,10);
\draw [->, >=Stealth] (12.25,10) -- (13.75,10);
\node [font=\LARGE] at (1,10.25) {\textit{\textbf{X}}};
\node [font=\LARGE] at (13,10.25) {\textit{\textbf{Y}}};
\end{circuitikz}
}%
\caption{PE-MHL architecture: a physics-encoded base model followed by $L$ neural network sub-models\blue{, each with three hidden layers of 50 ReLU neurons}.}
\label{fig:PEMHL}
\end{figure}

%% file: pictures/arch.tex
\begin{figure}[!ht]
\centering
\resizebox{0.45\textwidth}{!}{%
\begin{circuitikz}
\tikzstyle{every node}=[font=\Large]

\draw [short] (1.25,9.75) -- (3,9.75);
\draw [short] (1.25,5.75) -- (3,5.75);
\draw [short] (1.25,7.75) -- (3,7.75);

\node [font=\LARGE] at (0,9.75) {$1$};
\node [font=\LARGE] at (0,7.75) {$y(k-1)$};
\node [font=\LARGE] at (0,5.75) {$u(k-1)$};

\draw  (3.75,9.75) circle (0.75cm) node {\LARGE 
} ;

\draw  (3.75,7.75) circle (0.75cm) node {\LARGE 
} ;
\draw  (3.75,5.75) circle (0.75cm) node {\LARGE 
} ;
\node [font=\LARGE] at (3.5,12.25) {Input layer};

\draw  (7.5,11.75) circle (0.75cm) node {\LARGE 
} ;
\draw  (7.5,9.75) circle (0.75cm) node {\LARGE 
} ;

\draw  (7.5,6) circle (0.75cm) node {\LARGE 
} ;
\draw  (7.5,4) circle (0.75cm) node {\LARGE 
} ;
\draw [->, >=Stealth] (4.5,9.75) -- (6.75,11.5);
\draw [->, >=Stealth] (4.5,9.75) -- (6.75,9.75);
\draw [->, >=Stealth] (4.5,9.75) -- (6.75,6.25);
\draw [->, >=Stealth] (4.5,9.75) -- (6.75,4.25);
\draw [->, >=Stealth] (4.5,7.75) -- (6.75,11.5);
\draw [->, >=Stealth] (4.5,7.75) -- (6.75,6.25);
\draw [->, >=Stealth] (4.5,7.75) -- (6.75,4.25);
\draw [->, >=Stealth] (4.5,7.75) -- (6.75,9.75);
\draw [->, >=Stealth] (4.5,5.75) -- (6.75,11.5);
\draw [->, >=Stealth] (4.5,5.75) -- (6.75,9.75);
\draw [->, >=Stealth] (4.5,5.75) -- (6.75,6.25);
\draw [->, >=Stealth] (4.5,5.75) -- (6.75,4.25);
\draw [dashed] (7.5,8.75) -- (7.5,7);

\draw  (10.25,11.75) circle (0.75cm) node {\LARGE 
} ;
\draw  (10.25,9.75) circle (0.75cm) node {\LARGE 
} ;

\draw  (10.25,6) circle (0.75cm) node {\LARGE 
} ;
\draw  (10.25,4) circle (0.75cm) node {\LARGE 
} ;
\draw [dashed] (10.25,8.75) -- (10.25,7);

\draw  (13,11.75) circle (0.75cm) node {\LARGE 
} ;
\draw  (13,9.75) circle (0.75cm) node {\LARGE 
} ;

\draw  (13,6) circle (0.75cm) node {\LARGE 
} ;
\draw  (13,4) circle (0.75cm) node {\LARGE 
} ;
\draw [dashed] (13,8.75) -- (13,7);

\draw  (16.75,7.75) circle (0.75cm) node {\LARGE 
} ;
\draw [->, >=Stealth] (8.25,11.75) -- (9.5,11.75);
\draw [->, >=Stealth] (8.25,11.75) -- (9.5,9.75);
\draw [->, >=Stealth] (8.25,11.75) -- (9.5,6.25);
\draw [->, >=Stealth] (8.25,11.75) -- (9.5,4);
\draw [->, >=Stealth] (8.25,9.75) -- (9.5,9.75);
\draw [->, >=Stealth] (8.25,6) -- (9.5,6);
\draw [->, >=Stealth] (8.25,4) -- (9.5,4);
\draw [->, >=Stealth] (8.25,9.75) -- (9.5,11.75);
\draw [->, >=Stealth] (8.25,9.75) -- (9.5,6);
\draw [->, >=Stealth] (8.25,6) -- (9.5,4);
\draw [->, >=Stealth] (8.25,4) -- (9.5,6);
\draw [->, >=Stealth] (8.25,4) -- (9.5,10);
\draw [->, >=Stealth] (8.25,6) -- (9.5,10);
\draw [->, >=Stealth] (8.25,6) -- (9.5,11.75);
\draw [->, >=Stealth] (11,9.75) -- (12.25,9.75);
\draw [->, >=Stealth] (11,6) -- (12.25,6);
\draw [->, >=Stealth] (11,4) -- (12.25,4);
\draw [->, >=Stealth] (11,11.75) -- (12.25,11.75);
\draw [->, >=Stealth] (11,11.75) -- (12.25,9.75);
\draw [->, >=Stealth] (11,11.75) -- (12.25,6);
\draw [->, >=Stealth] (11,11.75) -- (12.25,4);
\draw [->, >=Stealth] (11,9.75) -- (12.25,11.75);
\draw [->, >=Stealth] (11,9.75) -- (12.25,6);
\draw [->, >=Stealth] (11,9.75) -- (12.25,4);
\draw [->, >=Stealth] (11,6) -- (12.25,9.75);
\draw [->, >=Stealth] (11,6) -- (12.25,11.75);
\draw [->, >=Stealth] (11,6) -- (12.25,4);
\draw [->, >=Stealth] (11,4) -- (12.25,9.75);
\draw [->, >=Stealth] (11,4) -- (12.25,6);
\draw [->, >=Stealth] (11,4) -- (12.25,11.75);
\draw [->, >=Stealth] (17.5,7.75) -- (19.25,7.75);
\node [font=\LARGE] at (20,7.75) {$\hat{y}(k)$};
\draw [->, >=Stealth] (13.75,9.75) -- (16,7.75);
\draw [->, >=Stealth] (13.75,6) -- (16,7.75);
\draw [->, >=Stealth] (13.75,4) -- (16,7.75);
\draw [->, >=Stealth] (13.75,11.75) -- (16,7.75);
\node [font=\huge] at (10,13.75) {Hidden layers (3 $\times$ 300 ReLU)};
\node [font=\LARGE] at (16.75,12.25) {Output layer};
\node [font=\LARGE] at (3.75,2.25) {$x \in R^3$};
\node [font=\Large] at (7.2,2.25) {$W^{1}\in R^{300\times300}$};
\node [font=\Large] at (10.35,2.25) {$W^{2}\in R^{300\times300}$};

\node [font=\Large] at (16.75,2.25) {$y^{3}\in R^{1\times300}$};
\node [font=\Large] at (13.5,2.25) {$W^{3}\in R^{300\times300}$};
\end{circuitikz}
}%
\caption{Architecture of the monolithic baseline network\blue{: inputs $[1,\,y(k-1),\,u(k-1)]$, three hidden layers of 300 ReLU neurons, linear output}.}
\label{fig:nn_architecture}
\end{figure}

%% file: pictures/aero.tex
\begin{figure}[!ht]
\centering
\resizebox{0.45\textwidth}{!}{%
\begin{circuitikz}
\tikzstyle{every node}=[font=\LARGE]
\node [font=\LARGE] at (5,10.75) {u(t)};
\draw [short] (5.75,10.75) -- (7.5,10.75);
\draw [short] (7.5,11) -- (7.5,11.75);
\draw [->, >=Stealth] (7.5,11.75) -- (9.5,11.75);
\draw  (9.5,12.5) rectangle  node {\LARGE -1} (11.25,11.25);
\draw [->, >=Stealth] (11.25,11.75) -- (12.5,11.75);
\draw  (12.5,12) rectangle (14.5,11.5);
\draw  (13.5,11.75) circle (0.75cm) node {\LARGE M1} ;
\draw  (12.5,10) rectangle (14.5,9.5);
\draw  (13.5,9.75) circle (0.75cm) node {\LARGE M2} ;
\draw [short] (7.5,11) -- (7.5,9.75);
\draw [->, >=Stealth] (7.5,9.75) -- (12.5,9.75);
\draw  (15.25,11.25) rectangle  node {\large position encoder} (18.5,10.5);
\draw [->, >=Stealth] (18.5,10.9) -- (20.5,10.9);
\node [font=\LARGE] at (21.25,11) {y(t)};
\end{circuitikz}
}%
\caption{System model for aero 2 setup.}
\label{fig:aero_system_model}

\end{figure}